\documentclass{article}
\usepackage{float}

\usepackage[preprint,nonatbib]{neurips_2024}

\usepackage[utf8]{inputenc} 
\usepackage[T1]{fontenc}    
\usepackage{hyperref}       
\usepackage{url}            
\usepackage{booktabs}       
\usepackage{amsfonts}       
\usepackage{amsthm}         
\usepackage{nicefrac}       
\usepackage{microtype}      
\usepackage{xcolor}         

\usepackage[T1]{fontenc}
\usepackage{graphicx}
\usepackage{booktabs}
\usepackage{subcaption}
\usepackage{algorithm}
\usepackage{algorithmic}
\usepackage{tabularx}
\usepackage[misc]{ifsym}

\usepackage{mwe}
\usepackage{amssymb, amsmath}
\newtheorem{theorem}{Theorem}
\usepackage{xcolor}  
\usepackage{multicol}
\usepackage{hyperref}

\title{Unlearning-based sliding window for continual learning under concept drift}

\author{%
Michał Woźniak \\
Wroclaw University of Science and Technology\\
Wrocław, Poland\\
ORCID: 0000-0003-0146-4205\\
\texttt{michal.wozniak@pwr.edu.pl}\\
\AND
Marek Klonowski\\
Wroclaw University of Science and Technology\\
Wrocław, Poland\\
ORCID: 0000-0002-3141-8712\\
\texttt{marek.klonowski@pwr.edu.pl}\\
\AND
Maciej Mączynski\\
Wroclaw University of Science and Technology\\
Wrocław, Poland\\
ORCID: 0009-0006-0354-4507\\
\texttt{maciej.maczynski@pwr.edu.pl}\\
\AND
Bartosz Krawczyk\\
Rochester Institute of Technology\\
NY, USA\\
ORCID: 0000-0002-9774-0106\\
\texttt{bartosz.krawczyk@rit.edu}\\
}

\begin{document}

\maketitle

\begin{abstract}
Traditional machine learning assumes a stationary data distribution, yet many real-world applications operate on nonstationary streams in which the underlying concept evolves over time. This problem can also be viewed as task-free continual learning under concept drift, where a model must adapt sequentially without explicit task identities or task boundaries. In such settings, effective learning requires both rapid adaptation to new data and forgetting of outdated information. A common solution is based on a sliding window, but this approach is often computationally demanding because the model must be repeatedly retrained from scratch on the most recent data. We propose a different perspective based on machine unlearning. Instead of rebuilding the model each time the active window changes, we remove the influence of outdated samples using unlearning and then update the model with newly observed data. This enables efficient, targeted forgetting while preserving adaptation to evolving distributions. To the best of our knowledge, this is the first work to connect machine unlearning with concept drift mitigation for task-free continual learning. Empirical results on image stream classification across multiple drift scenarios demonstrate that the proposed approach offers a competitive and computationally efficient alternative to standard sliding-window retraining. Our implementation can be found at \href{https://anonymous.4open.science/r/MUNDataStream-60F3}{https://anonymous.4open.science/r/MUNDataStream-60F3}.

\end{abstract}

\section{Introduction}

Task-free continual learning~\cite{ye2022taskfree} addresses scenarios in which data arrive continuously in the form of a stream, and the underlying distribution may evolve over time. In such nonstationary environments, \emph{concept drift} can substantially degrade predictive performance unless the learner continuously adapts to newly emerging patterns while reducing the influence of outdated information~\cite{Ashrafee:2026}. Existing adaptation strategies are commonly based on explicit drift detection, online adaptation, or sliding-window retraining~\cite{cano2022rose,Krawczyk:2017}, where the model is updated using only the most recent portion of the stream. While effective, these mechanisms often require repeated retraining and careful window-size selection, which can increase computational cost and limit scalability.

At the same time, recent advances in \emph{machine unlearning}~\cite{nguyen2022surveymachineunlearning} have shown that it is possible to approximately remove the effect of selected training data from an already trained model without retraining it from scratch. However, unlearning has been studied almost exclusively in the context of privacy, safety, and data removal requests, rather than as a mechanism for adaptive learning under concept drift. 

\noindent \textbf{Research objective.} There exists a gap between the task-free continual learning and machine unlearning: although both are concerned with controlling the influence of past data on a model, their connection remains largely unexplored. To address this gap, we propose to reinterpret forgetting under concept drift through the lens of machine unlearning, using it as an efficient alternative to repeated sliding-window retraining.

\noindent \textbf{Main contributions.} To the best of our knowledge, this is the first work to investigate machine unlearning as a mechanism for mitigating concept drift in data streams. Our main contributions are as follows:
\begin{itemize}
    \item We introduce \emph{UIL} (\textbf{U}nlearned and \textbf{I}teratively trained c\textbf{L}assifier), a novel framework that mitigates concept drift by applying machine unlearning to simulate tailored forgetting of obsolete data, thereby providing an alternative to computationally demanding sliding-window retraining.
    
    \item We provide a theoretical analysis showing that, under the considered assumptions, unlearning outdated data followed by incremental adaptation can approximate the predictive behavior of retraining from scratch while requiring substantially lower computational cost.
    
    \item We conduct illustrative experiments on image classification benchmarks with several types of sudden concept drift, demonstrating the practical behavior of the proposed approach in controlled nonstationary settings.
\end{itemize}


    


\section{Related works}

\noindent \textbf{Learning from streaming data.} Recent works increasingly narrow the gap between \emph{data stream mining} and \emph{task-free continual learning} by studying learning from non-stationary streams without explicit task identities or boundaries. In task-free and online continual learning, Ye and Bors~\cite{ye2022taskfree}, Ghunaim et al.~\cite{ghunaim2023realtime}, and Huang et al.~\cite{huang2024las} formulate continual adaptation as sequential learning over evolving data, emphasizing the role of replay, realistic evaluation, and robustness to time-varying class priors. In parallel, recent stream-mining works such as MCD-DD~\cite{wan2024mcddd} and ReCDA~\cite{yang2024recda} study drift-aware representation learning and adaptation directly in streaming settings. Taken together, these works support a unified view in which concept-drifting streams can be interpreted as task-free continual learning processes governed by evolving distributions rather than predefined task switches~\cite{Ashrafee:2026}.

\noindent \textbf{Learning with sliding windows.} \emph{Sliding-window} mechanisms remain one of the most practical ways to localize adaptation to the most relevant recent history. In the data-stream literature, ROSE~\cite{cano2022rose} uses class-wise windows to maintain robust online ensembles under joint class imbalance and drift, while MCD-DD~\cite{wan2024mcddd} employs window-based drift monitoring in high-dimensional streams. Recent work on feature streams also combines windowing with explicit change analysis, as in Zhou et al.~\cite{zhou2024concept}. In online continual learning, Huang et al.~\cite{huang2024las} use a batch-wise sliding-window estimator to track time-varying class priors, showing that short-horizon statistics can substantially improve stability in non-stationary streams. 


\noindent \textbf{Unlearning approaches.} \emph{Machine unlearning}~\cite{nguyen2022surveymachineunlearning} aims to remove the influence of selected data points, classes, or concepts from a trained model without retraining it from scratch. This capability is important in settings where data must be removed due to annotation errors, revoked usage rights, privacy concerns, adversarial contamination, or harmful content, while full retraining may be computationally prohibitive or impossible because the original data are no longer fully accessible. Existing methods can be broadly divided into \emph{exact} and \emph{approximate} approaches. Exact methods, such as SISA~\cite{DBLP:conf/sp/BourtouleCCJTZL21} and ARCANE~\cite{DBLP:conf/ijcai/YanLG0L022}, provide stronger removal guarantees but often require specialized procedures or substantial overhead. Approximate methods~\cite{DBLP:journals/tifs/ChundawatTMK23,Guo:2020} are more practical and scalable, but trade computational efficiency for imperfect removal. Since the notion of \emph{correct} unlearning remains difficult to define and evaluate rigorously~\cite{DBLP:conf/uss/ThudiJSP22}, we do not study certified removal guarantees here. Instead, following recent critiques of unlearning as a universal deletion mechanism~\cite{Cooper:2025}, we treat approximate unlearning as a practical tool for \emph{simulated forgetting} in continual learning under concept drift.

\section{Unlearned and Iteratively Trained classifier}

\subsection{Problem statement}

We present nonstationary stream learning as a \emph{task-free continual learning} problem. The data stream is given as a sequence of chunks \(DS=(D_1,D_2,\dots)\), where each chunk \(D_k=\{(x_i^k,y_i^k)\}_{i=1}^{m}\) contains \(m\) labeled samples drawn from a time-dependent distribution \(P_k(X,Y)\). The learner receives chunks sequentially, without access to task labels or explicit task boundaries, and updates the model parameters $\theta$ online according to:
\begin{equation}
\theta_{k+1}=\mathcal{A}(\theta_k,D_k),
\end{equation}
\noindent where $\mathcal{A}$ stands for the incremental update rule. The goal is to maintain low expected loss
$\mathbb{E}_{(x,y)\sim P_k}[\ell(f_\theta(x),y)]$ on the current stream distribution for a selected loss function $\ell$,
\noindent while preserving useful information from recent history and adapting to new concepts. Concept drift occurs whenever the stream distribution changes over time, i.e.,
\begin{equation}
P_k(X,Y)\neq P_{k+\tau}(X,Y), \qquad \text{for some } \tau>0,
\end{equation}
\noindent which captures changes in \(P(X)\), \(P(Y\mid X)\), or both. Under this view, concept-drifting stream learning is a task-free continual adaptation problem rather than a sequence of explicitly separated tasks.


\subsection{Proposed approach}

One popular approach in stream learning is the sliding window method, which ensures that only objects from a selected time window are used to build the model. Despite the many algorithms that use this approach, which mainly focus on selecting the window appropriately, including adapting its length to changes in the distribution, a significant problem remains: they require retraining the model on a new data window. The proposed approach significantly reduces the cost of updating learning models by leveraging machine unlearning and incremental retraining.

Let us present the main idea of \emph{UIL} algorithm by comparing it with the classic Sliding Window (\emph{SW}) algorithm. \emph{SW} learns on a data window consisting of $L$ data chunks. Until the full window is received, the classifier trains incrementally. Once the window is full and a new chunk arrives, the model is trained from scratch on the last $L$ data chunks.
\emph{UIL} behaves the same until the window buffer is filled. However, when a new data chunk arrives, instead of retraining the model from scratch, it unlearns the oldest chunk and then retrains on the new one. The pseudocodes of \emph{SW} and \emph{UIL} are shown in Alg. 1 and 2, respectively, with the lines that differ between the two algorithms highlighted. In addition, to better illustrate how \emph{UIL} works, we present its detailed pipeline in Fig. \ref{fig:uil}.

\begin{figure}
    \centering
    \includegraphics[width=0.8\linewidth]{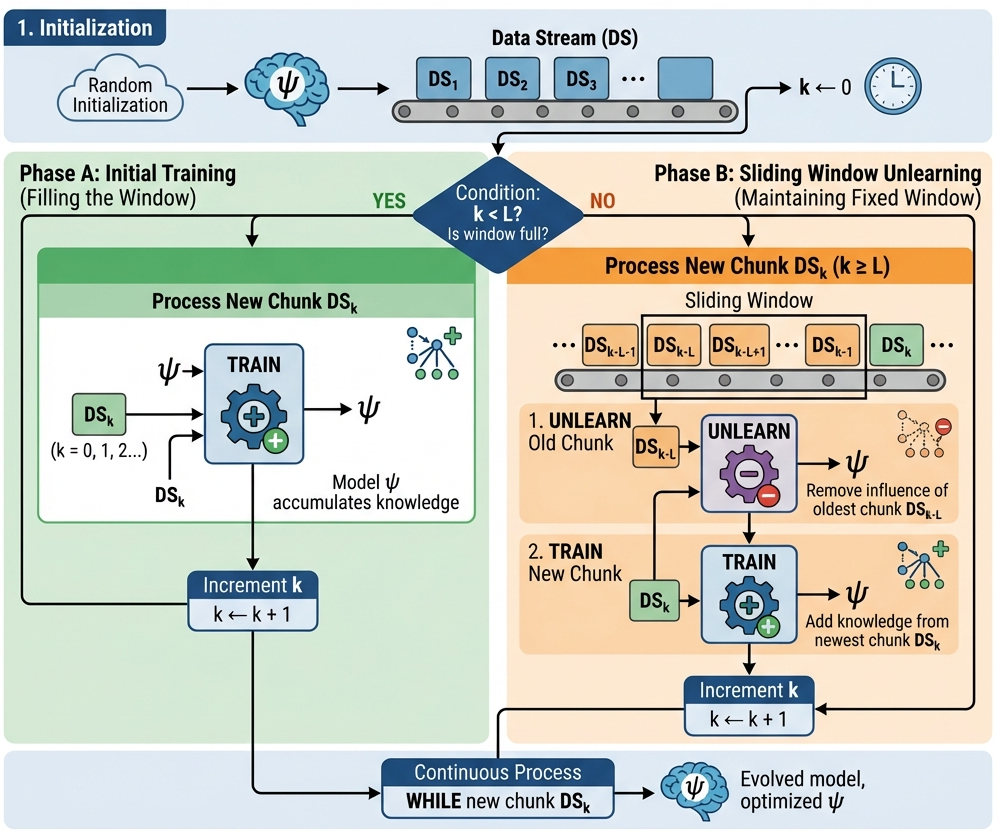}
    \caption{The idea of UIL (Unlearned and Iteratively trained cLassifier)}
    \label{fig:uil}
\end{figure}

\begin{algorithm}
\caption{(\emph{UW}) Training a stream classifier using a sliding window}
\label{alg1:sliding}
\begin{algorithmic}[1]

\REQUIRE window length $L$, data stream $DS = DS_1, DS_2, \dots $
\ENSURE trained model $\psi$

\STATE $k \leftarrow 0$
\STATE Initate randomly $\psi$

\WHILE{new chunk $DS_k$}
    \IF{$k < L$}
        \STATE $\psi \leftarrow \text{Train}(\psi, DS_k)$
    \ELSE
        \STATE \textcolor{red}{Reset model $\psi$}
        \STATE \textcolor{red}{$\psi \leftarrow \text{Train}\big(DS_{k-L+1}, \dots , DS_k\big)$}
    \ENDIF
    \STATE $k \leftarrow k + 1$
\ENDWHILE
\end{algorithmic}
\end{algorithm}

\begin{algorithm}
\caption{(\emph{UIL}) Training a  stream classifier using a machine unlearning}
\label{alg2:unlearning}
\begin{algorithmic}[1]

\REQUIRE window length $L$, data stream $DS = DS_1, DS_2, \dots $
\ENSURE trained model $\psi$

\STATE $k \leftarrow 0$
\STATE Initiate randomly $\psi$

\WHILE{new chunk $DS_k$}
    \IF{$k < L$}
        \STATE $\psi \leftarrow \text{Train}(\psi, DS_k)$
    \ELSE
        \STATE \textcolor{red}{$\psi \leftarrow \text{Unlearn}\big(\psi, DS_{k-L})$}
        \STATE \textcolor{red}{$\psi \leftarrow \text{Train}\big(\psi, DS_k\big)$}
    \ENDIF
    \STATE $k \leftarrow k + 1$
\ENDWHILE
\end{algorithmic}
\end{algorithm}



\section{Analytical evaluation}




We should answer the following research questions to assess the \emph{UIL} quality in relation to \emph{SW}:
\begin{itemize}
    \item[RQ1] What is \emph{UIL} predictive performance in relation to \emph{SW} when data distribution changes (concept drift appears)?
    \item[RQ2]  What is \emph{UIL} predictive performance in relation to \emph{SW} in the case of a stationary distribution, i.e., how will catastrophic forgetting affect the possible deterioration of the model, and is it possible to formulate the conditions that should be met for the model to be stable?
    \item[RQ3] What is the \emph{UIL} computational complexity compared to the \emph{SW}?
\end{itemize}

Firstly, let us present the basic definitions and assumptions.

The processed data stream $DS=(D_1, D_2, ...)$ is a sequence of chunks of equal length $m$, called the chunk size. The $i$th chunk is a  sequence of pairs $D_i=((x_1^{i}, y_1^i), (x_2^i, y_2^i), ...,(x_m^i, y_m^i))$, where $x_k^i$ represents the feature vector values of the $k$th sample in the $i$th chunk, while $y_k^i$ denotes its corresponding label.

\emph{SW} uses a window $W_t=(D_{t+1}, D_{t+2}, ..., D_{t+L})$ that consists of $L$ data chunks.

Let the model (classifier) be represented by its parameters $\theta$. Let's $\text{Train}$ denote the training algorithm that can train the model from scratch, e.g., on the data window $W_t$ 
\begin{equation*}
    \theta_t \leftarrow \text{Train}(W_t)
\end{equation*}
\noindent
and can also update the model $\theta$ with a given data chunk $D_t$
\begin{equation*}
    \theta_t \leftarrow \text{Train}(\theta_{t-1}, D_t)~.
\end{equation*}

Let $\text{Unlearn}$ denote the unlearning algorithm. 
\begin{equation*}
    \theta_t \leftarrow \text{Unlearn}(\theta_{t-1}, D_t).
\end{equation*}

Let $\ell(\theta) := \ell(\cdot, \theta)$ be the loss function. We assume that $\ell(\cdot, \theta)$ is $\mu$-strongly convex with respect to $\theta$, i.e., for all $\theta, \theta'$:
\begin{equation*}
    \ell(\theta') \geq \ell(\theta) + \langle \nabla \ell(\theta), \theta' - \theta \rangle + \frac{\mu}{2} \|\theta' - \theta\|^2.
\end{equation*}
Additionally, we assume that the loss function is $\beta$-smooth. A common consequence of this, often used in analysis, is that for all $\theta, \theta'$:
\begin{equation*}
    \ell(\theta) - \ell(\theta') \leq \langle \nabla \ell(\theta'), \theta - \theta' \rangle + \frac{\beta}{2} \|\theta - \theta'\|^2.
\end{equation*}
These are widely adopted assumptions (see e.g.,~\cite{Rakhlin:2012}).

\subsection{Evaluating the performance for nonstationary data stream}

To answer the first research question about the \emph{UIL} behaviour for the nonstationary data stream, let us consider the problem of the sudden drift, i.e, when the data distribution changes rapidly. To simplify the consideration, we assume that all samples in a given data chunk are from the same distribution. 

Let the first $L$ chunks be drawn from the distribution $P_{old}$ and the successive chunks from $P_{new}$.

\noindent Initial State ($t=0$): The active window is $W_0 = (D_1, D_2, \dots  ,D_L)$, and all chunks are drawn from the distribution $P_{old}$. The base model $\theta_0$ is trained on the whole $W_0$.

\noindent Transition Step $t$ (for $t=1 \dots L$): The active window $W_t=(D_{t+1}, D_{t+2}, ..., D_{t+L})$. The sliding window \emph{SW} trains the model $\theta_t^*$ that is perfectly retrained on $W_t$. \emph{UIT} unlearns $D_t$ from $\theta_0$ and and incrementally update model $\theta_{t-1}$ by $D_{t+L}$ at step $t$, i.e., 

\begin{align*}
    \theta_t &\leftarrow \text{Train}(\text{Unlearn}(\theta_{t-1}, D_t), D_{t+L}), \\
    \theta_t^* & \leftarrow \text{Train}(W_t). 
\end{align*}

Note that every time \emph{UIL}  may introduce some error induced by imperfect unlearning (\text{Unlearn} procedure), inaccuracies, and catastrophic forgetting
(\text{Train} procedure). 

The first error arises from the fact that, in practice, effective unlearning of $D_t$ requires the use of approximate unlearning algorithms, which inevitably lead to a loss of model accuracy. This stems, among other factors, from the need to employ approximate Hessian inversion, since the exact algorithm is computationally inefficient.  As a result, the model weights deviate slightly from the values they would reach under perfect retraining without $D_t$. We denote this inaccuracy as $\epsilon_{\text{unlearn}}$.

After unlearning $D_t$, we gradually train the model on the new data chunk $D_{L+t}$. However, during the optimization procedure, gradients based on the new chunk $D_{L+t}$ are analysed, not those on the retained chunks ($D_{t+1}, \dots ,D_{L+t-1}$). The new data causes a slight deviation of the model weights from the optimal configuration for the retained chunks. The model forgets some of the retained knowledge associated with the remaining chunks in $W_t$ (catastrophic forgetting). This parameter deviation is $\epsilon_{\text{forgetting}}$.
We will bound the accumulated error by~$\delta$:



\begin{equation}
\label{eq:delta}
\|\theta_t - \theta_t^*\|_2 \leq \|\theta_{t-1} - \theta_{t-1}^*\|_2 + \delta.    
\end{equation}
\noindent

\noindent
Clearly $\delta = \epsilon_{\text{unlearn}} + \epsilon_{\text{forgetting}}$.


\begin{theorem}

Let $\mathbb{E}_{z \sim P_{new}}[\ell(z, \theta)]$ be the expected loss on the new distribution $P_{new}$ for a model $\theta$. After $L$ steps, the window consists entirely of new data ($W_L = (D_{L+1} \dots D_{2L})\text{ and } D_i \sim P_{new} \text{ for } i =L+1, \ldots 2L. $).  We can observe that  the expected loss for  \emph{UIL} (i.e. $\theta_L$) is bounded as follows:

\begin{equation}
\mathbb{E}_{z \sim P_{new}}[\ell(z, \theta_L)]
 \leq 
\mathbb{E}_{z \sim P_{new}}[\ell(z, \theta_L^*)]
+ \frac{\beta}{2}(L \cdot \delta)^2 ~. 
\end{equation}

\end{theorem}

\begin{proof} 
At $t=1$, the objective function for the ideal model shifts from minimizing $P_{old}$ entirely, to minimizing a mixture: $\frac{L-1}{L} P_{old} + \frac{1}{L} P_{new}$.
When the unlearning algorithm (e.g., a Newton step) removes $D_1$, it subtracts the gradient of $D_1$ multiplied by an approximation to the inverse Hessian. Because $D_1 \sim P_{old}$, this mathematical operation deliberately weakens the model's confidence in the $P_{old}$ landscape. Analytically, the model is pushed out of the minimum of $P_{old}$ and begins sliding toward the minimum of $P_{new}$. The parameter distance between the proposed model $\theta_1$ and the perfectly retrained model (i.e., used \emph{SW}$)~\theta_1^*$ is bounded by the single-step error $\delta$ shown in eq. \ref{eq:delta}.

As $t$ progresses from $1$ to $L$, this process is repeated. At each step, \emph{UIL} unlearns another chunk of $P_{old}$ and learns another chunk of $P_{new}$. By step $L$, the influence of $P_{old}$ is to be entirely removed from the window.
However, because \emph{UIL} uses the previously adapted model $\theta_{t-1}$ as the starting point for step $t$, the approximation errors stack. Since 
$\|x + y\| \leq \|x\|+ \|y\|$ we may bound the accumulated parameter distance recursively:

\begin{equation}
\|\theta_t - \theta_t^*\|_2 \leq \|\theta_{t-1} - \theta_{t-1}^*\|_2 + \delta~.
\end{equation}

Unrolling this recursion from $t=1$ to $t=L$ (starting from $\theta_0 = \theta_0^*$, meaning zero initial error), we get the maximum accumulated parameter distance after a full window shift

\begin{equation}
\label{eq:accumulation}
\|\theta_L - \theta_L^*\|_2 \leq L \cdot \delta~.
\end{equation}

At step $L$, $W_L$ includes only samples drawn from $P_{new}$. Therefore, the ideal model $\theta_L^*$ is the optimal model.

Because, as we assumed, the loss function is $\beta$-smooth, we may apply the second-order Taylor expansion around the ideal model $\theta_L^*$

\begin{equation}
\begin{split}
\mathbb{E}_{z \sim P_{new}}[\ell(z, \theta_L)] 
&\leq \mathbb{E}_{z \sim P_{new}}[\ell(z, \theta^{*}_L)] \\
&\quad + \nabla \left( \mathbb{E}_{z \sim P_{new}}[\ell(z, \theta^{*}_L)] \right)^T (\theta_L - \theta_L^*) + \frac{\beta}{2} \|\theta_L - \theta_L^*\|_2^2~.
\end{split}
\end{equation}

\noindent 
Since $\theta^{*}_L$ is a model minimising the loss function, we may assume $$\nabla \left (\mathbb{E}_{z \sim P_{new}}[\ell(z, \theta^{*}_L)]\right ) \approx 0.$$ Finally, substituting the accumulated distance bound from eq.\ref{eq:accumulation} we get:

\begin{equation}
\mathbb{E}_{z \sim P_{new}}[\ell(z, \theta_L)]
\leq \mathbb{E}_{z \sim P_{new}}[\ell(z, \theta^{*}_L)] + \frac{\beta}{2} \|\theta_L - \theta_L^*\|_2^2~.
\end{equation}


\end{proof} 


 

\subsection{Evaluating performance for a stationary data stream}
Let us focus on the second research question, i.e., to evaluate \emph{UIL} predictive performance in relation to \emph{SW} in the case of a stationary distribution, i.e., how catastrophic forgetting can affect the possible deterioration of the model, and if it is possible to formulate the conditions that should be met for the model to be stable.


Thus, let's assume all chunks $D_1, \dots, D_{L+1}$ are drawn from the exact same stationary distribution $P$.
Because the underlying distribution does not change, the perfectly retrained model (by \emph{SW}) at any step $t$ essentially aims for the same global optimum $\theta_P^*$.
Therefore, the target does not move: $\theta_t^* \approx \theta_{t-1}^* \approx \theta_P^*$.

Even though the target $\theta_P^*$ does not change, our operations still introduce errors. Unlearning $D_t$ introduces $\epsilon_{\text{unlearn}}$, and incrementally training $D_{L+t}$ introduces $\epsilon_{\text{forgetting}}$.
If we rely strictly on the previous additive bound ($e_t \leq e_{t-1} + \delta$), the error would continually grow even when the data is perfectly stable. The model's weights would undergo a 
radnom perturbations slowly drifting away from $\theta_P^*$ purely due to algorithmic noise.

To keep the error at the same level (or minimize it), \emph{UIL} should pull the model back toward $\theta_P^*$ stronger than the unlearning step pushes it away. This restoring force comes from the incremental training step. Let's remain that  $\ell(.,\theta)$ is $\mu$-strongly convex and \emph{EIL} uses a learning rate $\eta$. According to standard optimization theory, applying SGD on a $\mu$-strongly convex loss landscape contracts the distance to the optimum by a factor of $(1 - \eta \mu)$, where $\eta \mu < 1$(see~\cite{Rakhlin:2012}).

Let $e_{t-1} = \|\theta_{t-1} - \theta_P^*\|_2$ be the parameter error at the previous step. During the unlearning step \emph{UIL} removes the influence of $D_t$. This perturbs the model, adding the approximation error $e_{\text{after\;unlearning}}\leq e_{t-1} + \epsilon_{\text{unlearn}}$. Then \emph{UIL} trains the model on the new chunk $D_{L+t}$ (which is drawn from the same distribution $P$). Because the data is from the true distribution, SGD actively pulls the weights toward $\theta_P^*$, shrinking the current error by $(1 - \eta \mu)$. However, because we only train on $D_{L+t}$ and not the whole window, the model suffers from \emph{catastrophic forgetting}, so the small forgetting variance $\epsilon_{\text{forgetting}}$ should be considered

\begin{equation*}
e_t \leq (1 - \eta \mu)( e_{t-1} + \epsilon_{\text{unlearn}}) + \epsilon_{\text{forgetting}}.
\end{equation*}


Taking this into consideration, we may formulate
 the \textit{stability condition}. To keep performance at the same level, the error at step $t$ must be no greater than the error at step $t-1$ ($e_t \leq e_{t-1}$). In other words



\begin{equation}
\label{eq:stability}
\eta \mu e_{t-1} \geq (1 - \eta \mu)\epsilon_{\text{unlearn}} + \epsilon_{\text{forgetting}}.    
\end{equation}

For the model to remain stable on a stationary data stream, the condition above must be met, i.e., the corrective pull of the new data ($\eta \mu e_{t-1}$) must be strictly greater than the combined noise introduced by unlearning and incremental forgetting. Let us also observe that during the incremental training phase on the new chunk $D_{L+t}$, you must use a high enough learning rate so that the model actively "corrects" the unlearning noise, but not so high that $\eta \mu > 1$ (which would cause SGD to diverge). Additionally, if the model starts perfectly ($\theta_0 = \theta_P^*$, so $e_0 = 0$), the error will initially grow. 

Additionally, we may find the maximum stable error as $t \to \infty$, i.e., the fixed point where the error stops growing. Based on eq. \ref{eq:stability} and observe that for 
$t \to \infty$ $e_t = e_{t-1} = e_{\infty}$, where

\begin{equation}
e_{\infty} = \frac{(1 - \eta \mu)\epsilon_{\text{unlearn}} + \epsilon_{\text{forgetting}}}{\eta \mu}~.
\end{equation}

\subsection{Runtime analysis}

Let us address the last research question regarding the computational cost of \emph{UIL} compared to \emph{SW}. In this analysis, we will not use asymptotic Big-$\mathcal{O}$ notation, because it may obscure the practical performance difference. Some works on the optimization and deep learning scaling literature, as \cite{Johnson:2013,Kaplan:2020}, evaluated computational complexity by formulating a cost function $\mathcal{C}$ that estimates the exact number of parameter gradient evaluations required.

For \emph{SW}, a model is trained from scratch on the active window $W_t$ using Stochastic Gradient Descent (SGD) for $E_{ret}$ epochs. The cost of processing a single data point (forward and backward pass) scales linearly with the number of parameters \cite{Bottou:2010}. Because the algorithm must pass over the entire window of size $L \cdot m$ (where $L$ is the number of chunks and $m$ is the chunk size) for $E_{ret}$ training epochs, the total computational cost is estimated as:

$$\mathcal{C}_{SW} \approx L \cdot m \cdot E_{ret} \cdot p~.$$

For \emph{UIL}, the adaptation process consists of unlearning the oldest data chunk using an approximate Newton step (or Influence Function) and then updating the model with the incoming data chunk using SGD for $E_{inc}$ epochs. To establish the total cost, we first decompose the unlearning phase into two distinct operations: the linear gradient computation and the constant-cost curvature approximation.

The first step in a Newton-based unlearning update is calculating the first-order gradient of the loss function for the outdated data chunk. This determines the precise parameter direction to subtract from the model. Computing this requires a standard forward and backward pass over the chunk, resulting in a cost of $m \cdot p$. The second step requires preconditioning the gradient using the inverse Hessian matrix to account for the loss landscape's curvature. Computing an exact inverse Hessian is $\mathcal{O}(p^3)$, which is intractable for deep networks. However, following the influence function approximations \cite{Koh:2017}, we can estimate the inverse Hessian-Vector Product using the Conjugate Gradient method. Crucially, this solver does not need to process the entire retained data window. Instead, it operates on a small, fixed random subsample (of size $k$) for $i$ iterations. Therefore, the cost of the inverse Hessian-Vector Product estimation is strictly bounded by $k \cdot i \cdot p$.

Finally, \emph{UIL} fine-tunes the model on the new data chunk using SGD for $E_{inc}$ epochs. The cost of this incremental step is $m \cdot E_{inc} \cdot p$. Notably, incremental training requires far fewer epochs than training from scratch ($E_{inc} \ll E_{ret}$) \cite{Castro:2018}. Summing these components yields the total computational cost for \emph{UIL}:

$$\mathcal{C}_{UIL} \approx \big( (1 + E_{inc}) \cdot m + k \cdot i \big) \cdot p~.$$



\subsection{Discussion}

Based on the analytical study, we can draw the following conclusions.

\noindent \textbf{Relation between error rate and window size.} Analyzing the behavior of \emph{UIL}, we may observe that the error increases quadratically with the size of the window $L$, which leads to the following observations: (i) if $L$ is small or $\delta$ (related to unlearning and catastrophic forgetting) is negligible, then \emph{UIL} may run indefinitely; and (ii) if $L$ is large, then the cumulative error may cause the model to deviate too much from the optimal state. Therefore, an anchor mechanism should be provided to periodically retrain the model and remove the influence of cumulative error. (RQ1 answered).

\noindent \textbf{Behavior under stationary continual streams.} Analyzed the \emph{UIH} on a stationary data stream, and we proved that the "noise" introduced by constantly unlearning and relearning data from the same distribution will not accumulate to infinity, and that \emph{UIL} is robust and stable even when the data stream is stationary. (RQ2 answered).

\noindent \textbf{Insights into computational complexity.} Although the computational complexity of both SW and UIL scales linearly with the number of model parameters p, the practical computational load differs significantly. The $ p$-multiplier is significantly smaller for UIL than for SW. More importantly, $\mathcal{C}_{UIL}$ is completely independent of the window size $L$. Thus, the computation cost \emph{SW} increases linearly with $L$. In contrast, the adaptation cost \emph{UIL} remains safely limited by the fixed size of the data chunk ($m$), so especially for large window sizes (measured in chunks), the difference in the number of calculations is significantly smaller for \emph{UIL}, i.e., $\mathcal{C}_{UIL} < \mathcal{C}_{SW}$. (RQ3 answered).


\noindent \textbf{On the impact of used optimizers.} While the computational cost bound is derived under the assumption of standard Stochastic Gradient Descent (SGD), modern deep learning implementations typically employ adaptive optimizers such as Adam~\cite{kingma:2015}. In the context of our framework, replacing SGD with Adam strictly amplifies the computational advantage. Because the unlearning step of \emph{UIL} systematically removes obsolete concepts while preserving generalized feature extractors, Adam's adaptive momentum terms immediately align with the new distribution. Consequently, the number of incremental epochs ($E_{inc}$) required for convergence drops drastically compared to full retraining with a random initialization ($E_{ret}$). Thus, in practical deployments, the empirical speedup of the proposed method significantly exceeds the baseline theoretical bounds.

\noindent \textbf{Insights into the used unlearning mechanism.} Additionally, it is worth noting that the theoretical considerations underlying the choice of an unlearning mechanism that uses iterative Hessian-Vector Products, which are highly precise and minimize $\epsilon_{unlearn}$. On the other hand, utilizing a diagonal Empirical Fisher Information Matrix~\cite{Golatkar:2019} eliminates the iterative solver, reducing the complexity to strictly $\mathcal{O}((m + i) \cdot p)$. However, the acceleration in computation comes at the cost of a less accurate curvature approximation, which negatively affects the unlearning error. It also increases the error threshold for catastrophic forgetting during learning phases under a stationary distribution. Nevertheless, the Empirical Fisher Information Matrix offers a practical solution for large-scale deep architectures where exact Hessian-vector product iterations are a bottleneck.

\section{Experimental evaluation}

We conducted a series of illustrative experiments to evaluate the impact of the selected parameters, such as learning and unlearning rates and window length, on the \emph{UIL} behaviour, particularly, we tried to answer the following research questions:

\begin{itemize}
    \item[RQ1] Is \emph{UIL} comparable to \emph{SW}, especially in terms of computational time and recovery analysis?
    \item[RQ2] Can the find a parameter setting to find a stable \emph{IIT} performance?
\end{itemize}

\subsection{Experimental set-up}

\begin{figure}
    \centering
    \includegraphics[width=0.8\linewidth]{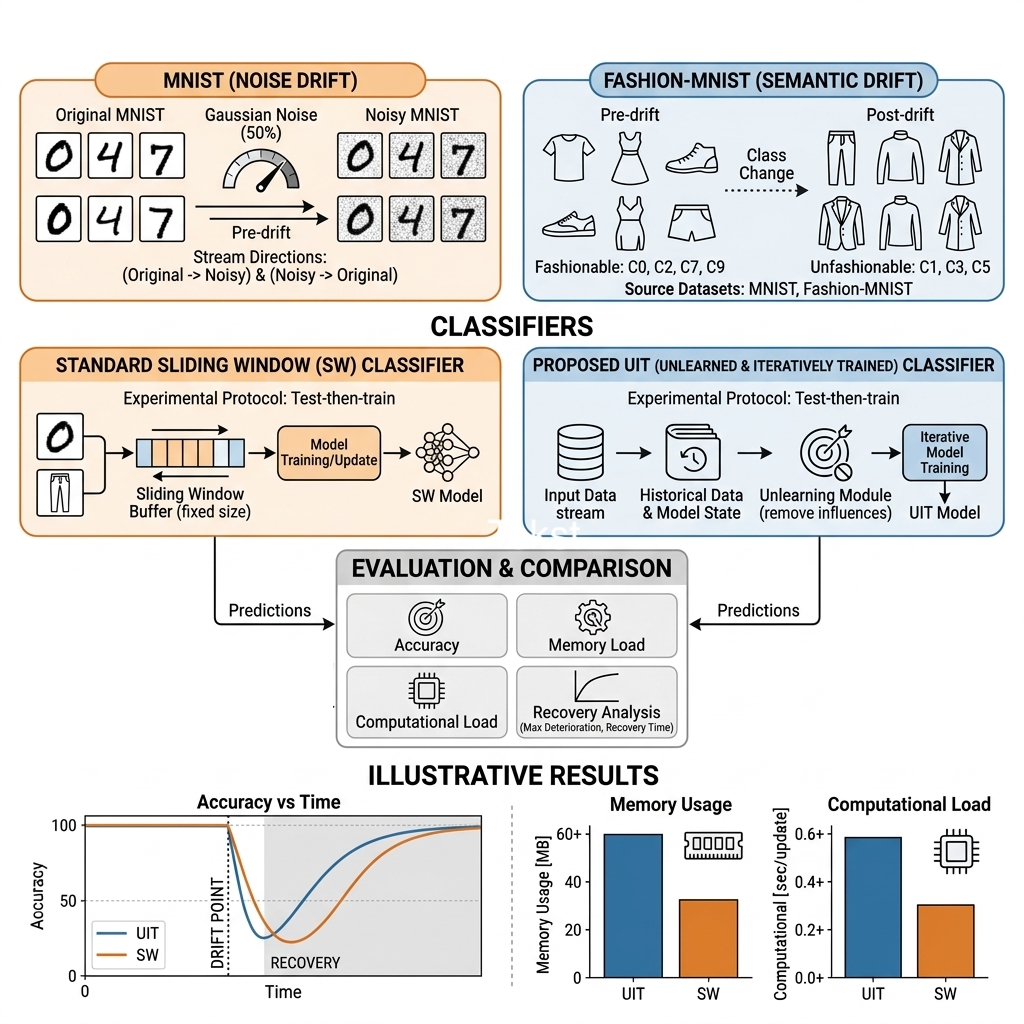}
    \caption{The proposed experimental pipeline.}
    \label{fig:pipeline}
\end{figure}    

\subsubsection{Data streams}
Because there is almost no image data stream, simulating an incremental environment requires processing existing benchmark datasets. Therefore, to evaluate the proposed methods, we created several data streams from well-known benchmark datasets – MNIST \cite{mnistdataset} and Fashion-MNIST \cite{Xiao:2017}. 
For the MNIST dataset, we first prepared a noisy version by adding Gaussian noise with 50\% intensity to the images. To obtain data streams, objects were initially randomly selected from the original MNIST and then from the noisy MNIST. Then we repeated the experiment in reverse order: noisy images appeared first in the stream, followed by noise-free images after the drift occurred.
For Fashion-MNIST, we applied the so-called semantic drift. We assigned clothing types to the fashionable and unfashionable categories. Firstly, the first class contains objects from classes 0, 2, 4, 7, and 9. After the appearance of the drift, it contains classes 1, 3, and 5.
The experimental pipeline is presented in Fig. \ref{fig:pipeline}.

\subsubsection{Metrics.} We use the following evaluation metrics: (i) accuracy, (ii) time per chunk to evaluate time cost, (iii) Data [MB per batch], (iv) recovery time, (v) maximal deterioration; and (vi) mean deterioration. We use recovery analysis \cite{Shaker2015} methods in order to compare how Sliding Window and implemented \emph{UIL} algorithm adapted to changes in data drift.

\subsubsection{Evaluation details.} For all datasets, we use ResNet18 architecture \cite{He:2016}. We searched for hyperparameters for \emph{UIL} and used the resulting values for all experiments (details are included in the repository related to this paper). As the base unlearning algorithm we used Certified Data Removal~\cite{Guo:2020}. All experiments were based on the \emph{Test-Then-Train}~\cite{Krawczyk:2017} evaluation protocol. All implementations were done using Python environment with scikit-learn~\cite{scikit-learn} and stream learn~\cite{ksieniewicz2020stream} packages.






\begin{figure}[ht]
\centering

\begin{subfigure}{0.32\textwidth}
    \centering
    \includegraphics[width=\linewidth]{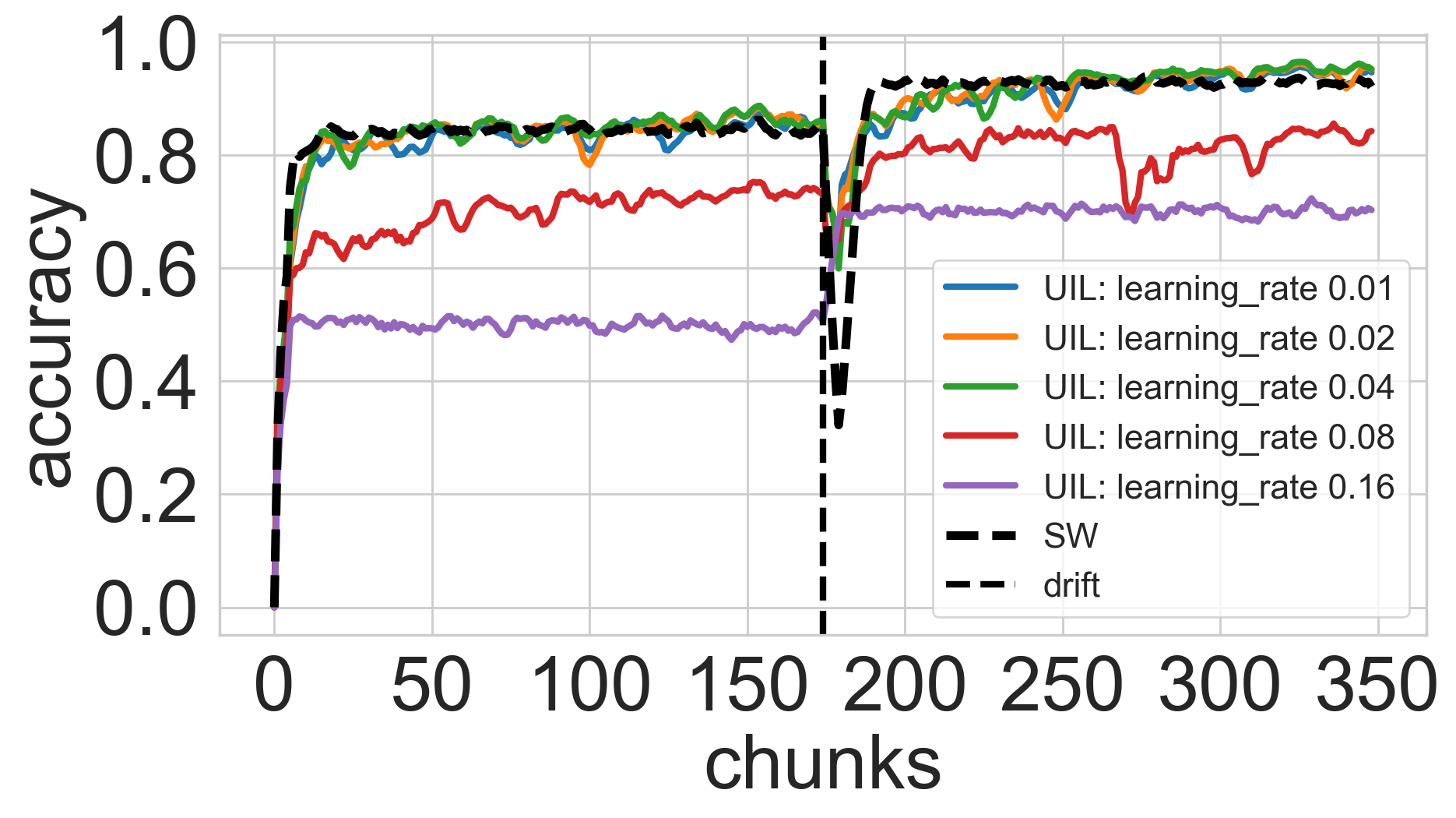}
    \caption{}
\end{subfigure}
\hfill
\begin{subfigure}{0.32\textwidth}
    \centering
    \includegraphics[width=\linewidth]{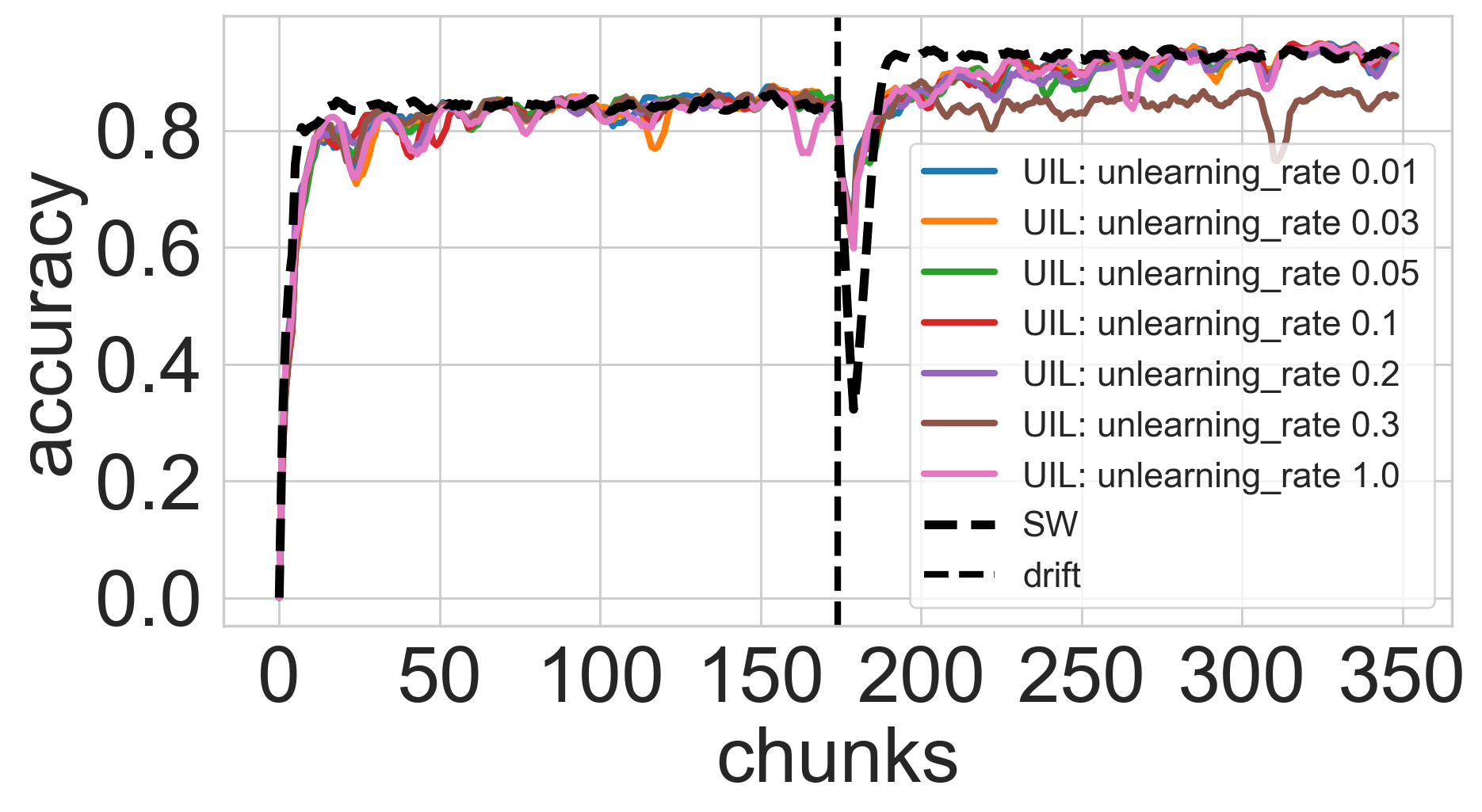}
    \caption{}
\end{subfigure}
\hfill
\begin{subfigure}{0.32\textwidth}
    \centering
    \includegraphics[width=\linewidth]{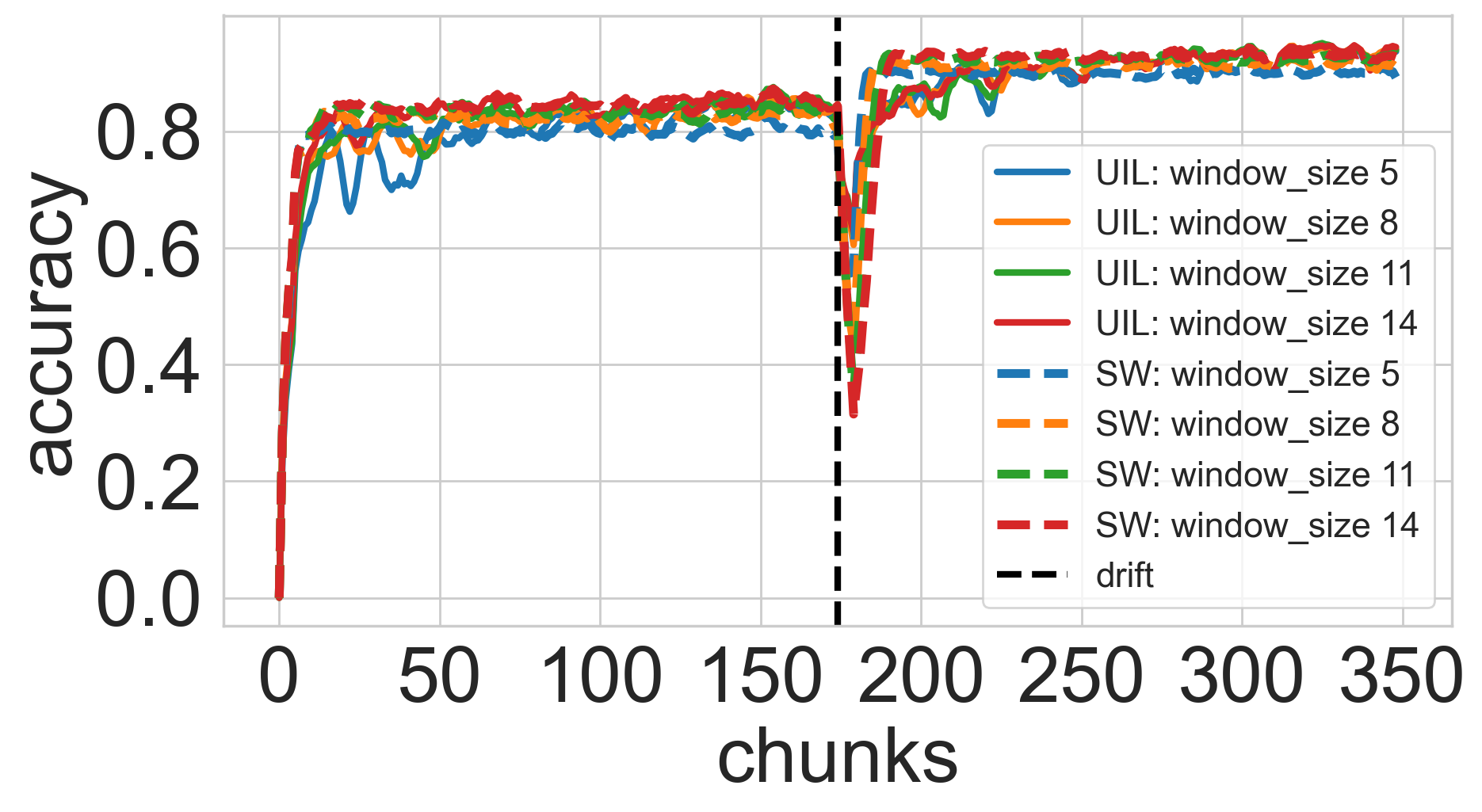}
    \caption{}
\end{subfigure}

\vspace{0.5cm}
\begin{subfigure}{0.32\textwidth}
    \centering
    \includegraphics[width=\linewidth]{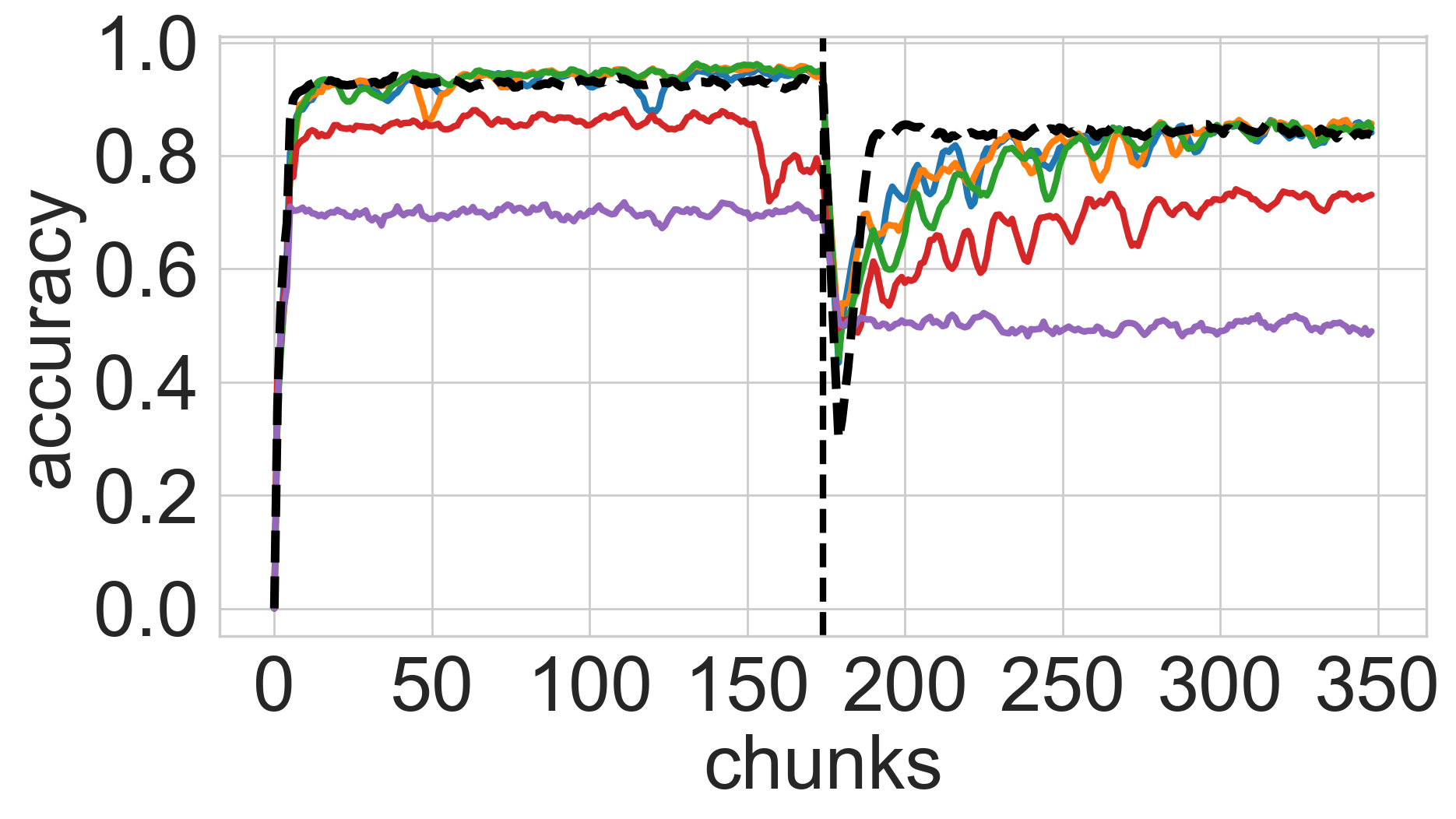}
    \caption{}
\end{subfigure}
\hfill
\begin{subfigure}{0.32\textwidth}
    \centering
    \includegraphics[width=\linewidth]{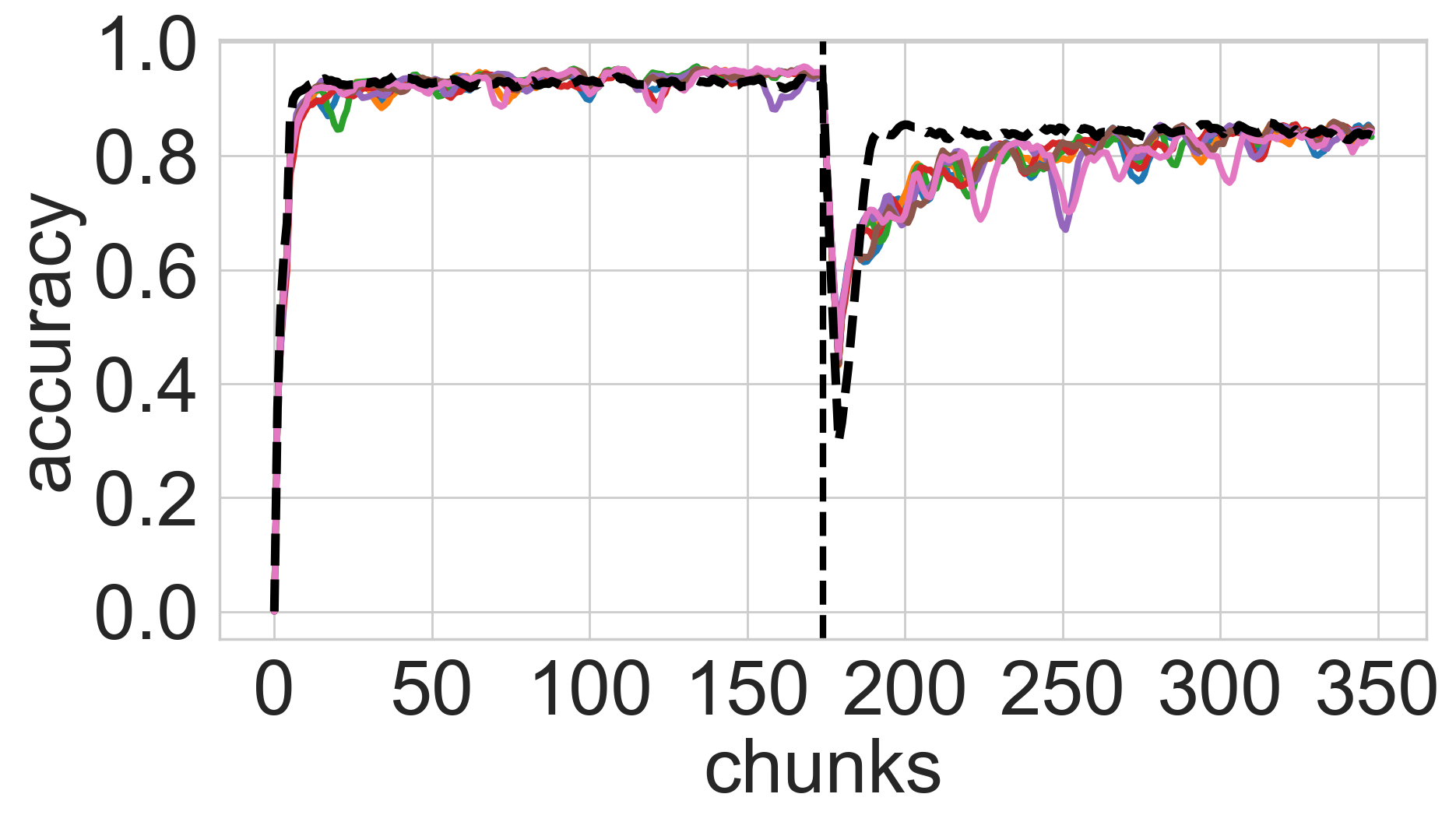}
    \caption{}
\end{subfigure}
\hfill
\begin{subfigure}{0.32\textwidth}
    \centering
    \includegraphics[width=\linewidth]{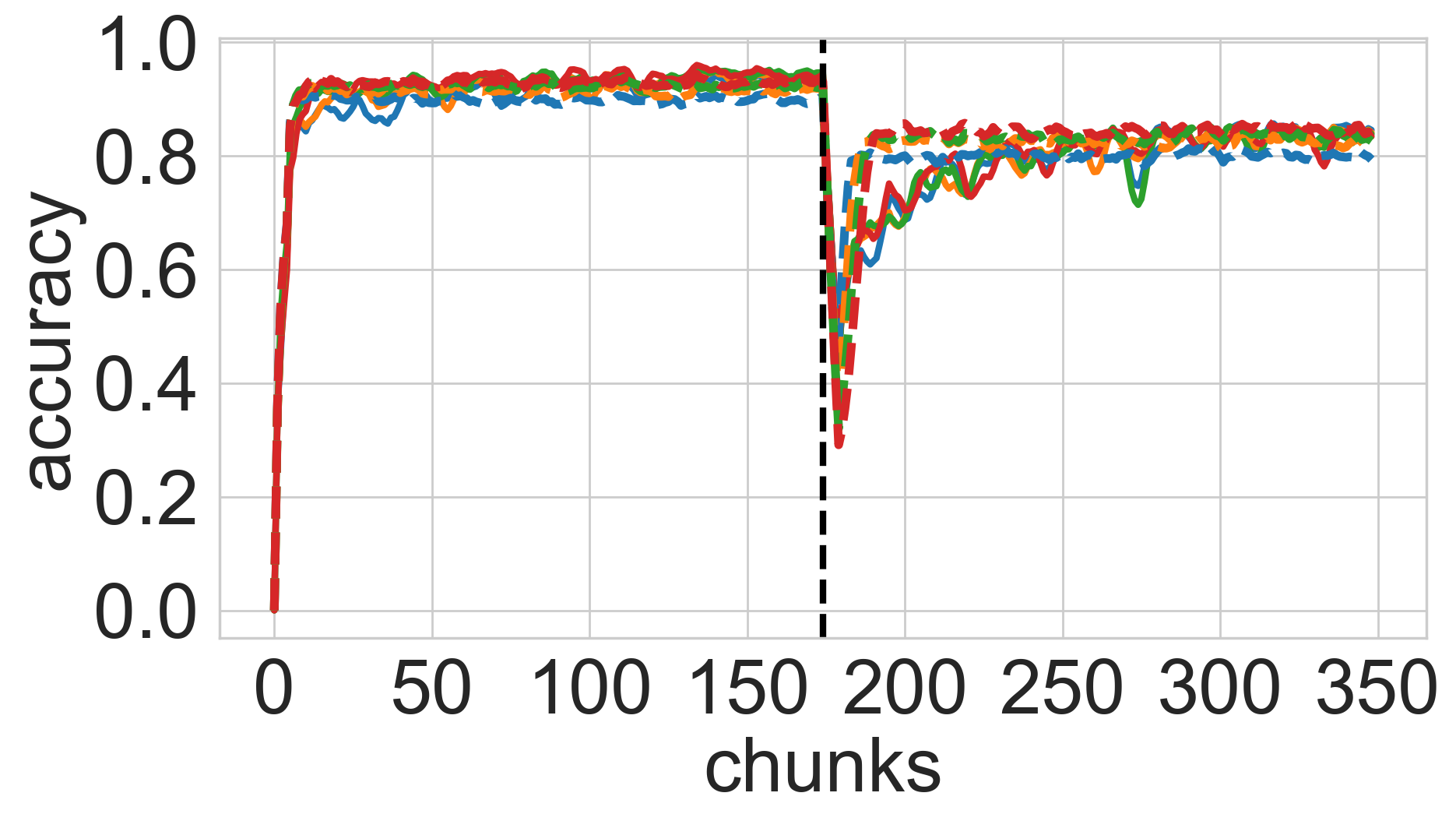}
    \caption{}
\end{subfigure}

\caption{Comparison of \emph{SW} and \emph{UIL} on semantic Fashion MNIST data stream: a, c, e) [0, 2, 4, 7] → [1, 3, 5], and  b, d, f) [0, 2, 4, 7] → [1, 3, 5]}
\label{fig:fashion_accuracy}
\end{figure}

\begin{figure}[!ht]
\centering

\begin{subfigure}{0.31\textwidth}
\centering
\includegraphics[width=\linewidth]{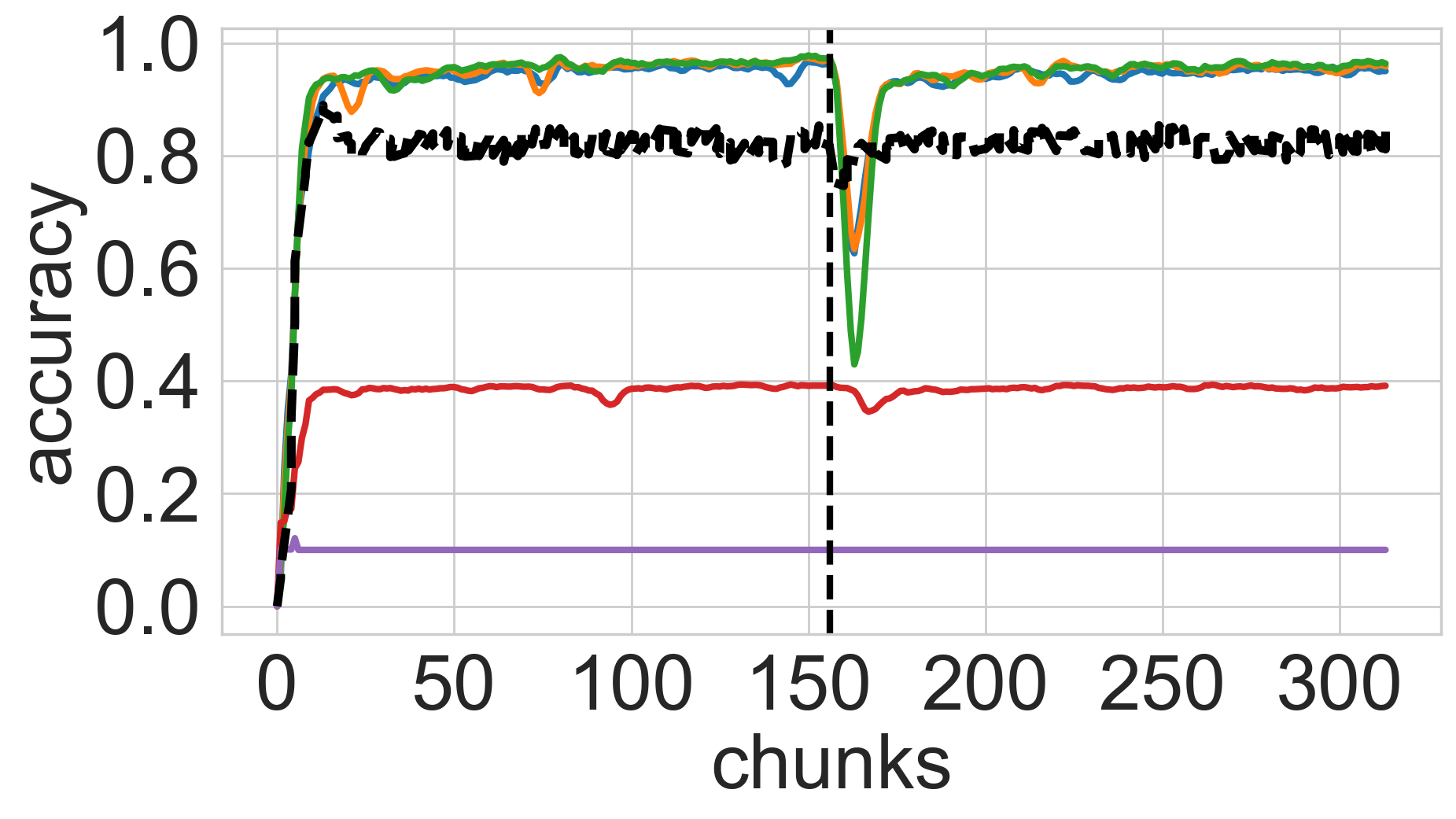}
\caption{}
\end{subfigure}
\hfill
\begin{subfigure}{0.31\textwidth}
\centering
\includegraphics[width=\linewidth]{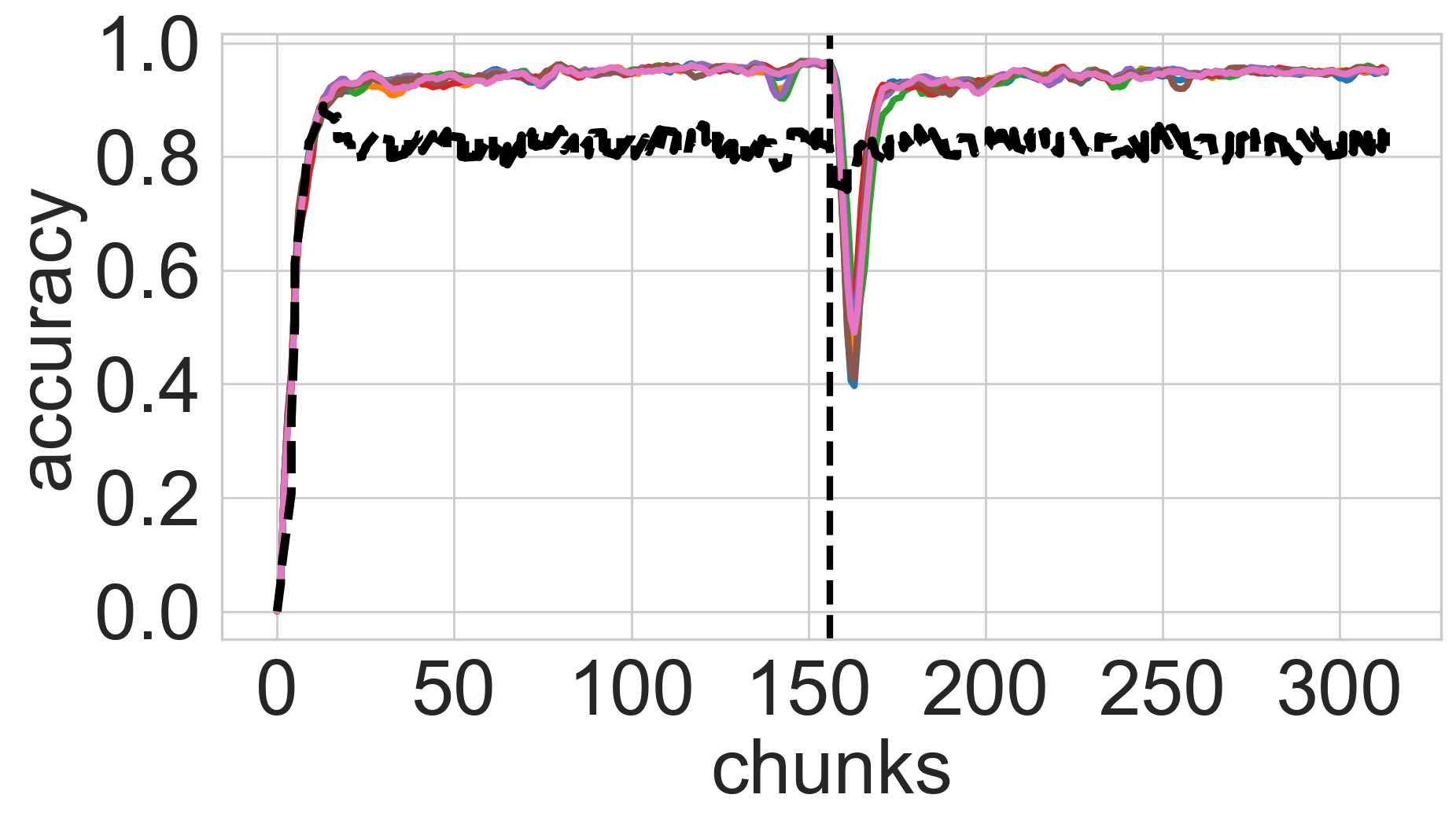}
\caption{}
\end{subfigure}
\hfill
\begin{subfigure}{0.31\textwidth}
\centering
\includegraphics[width=\linewidth]{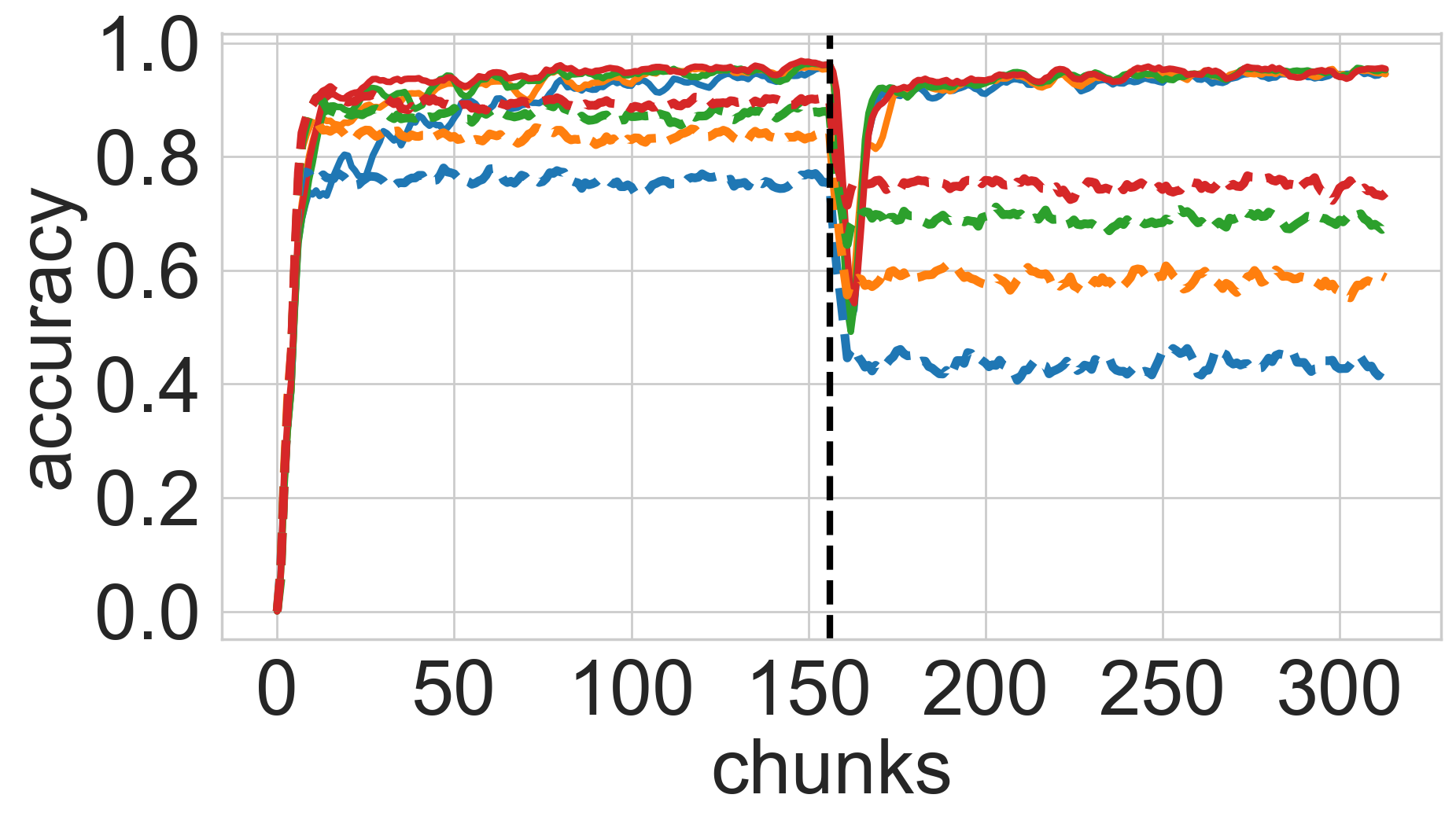}
\caption{}
\end{subfigure}

\vspace{0.5cm}
\begin{subfigure}{0.31\textwidth}
\centering
\includegraphics[width=\linewidth]{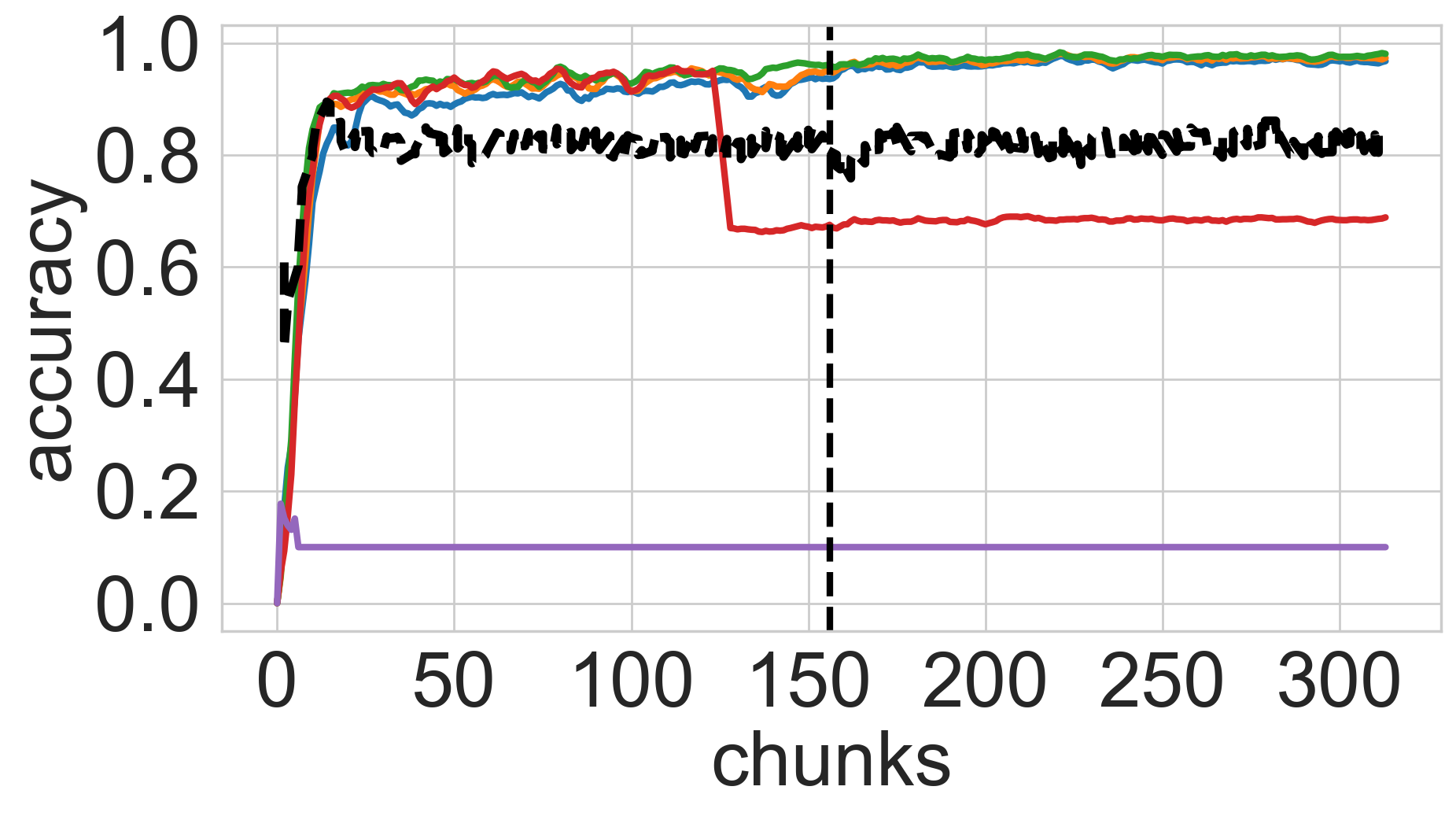}
\caption{}
\end{subfigure}
\hfill
\begin{subfigure}{0.31\textwidth}
\centering
\includegraphics[width=\linewidth]{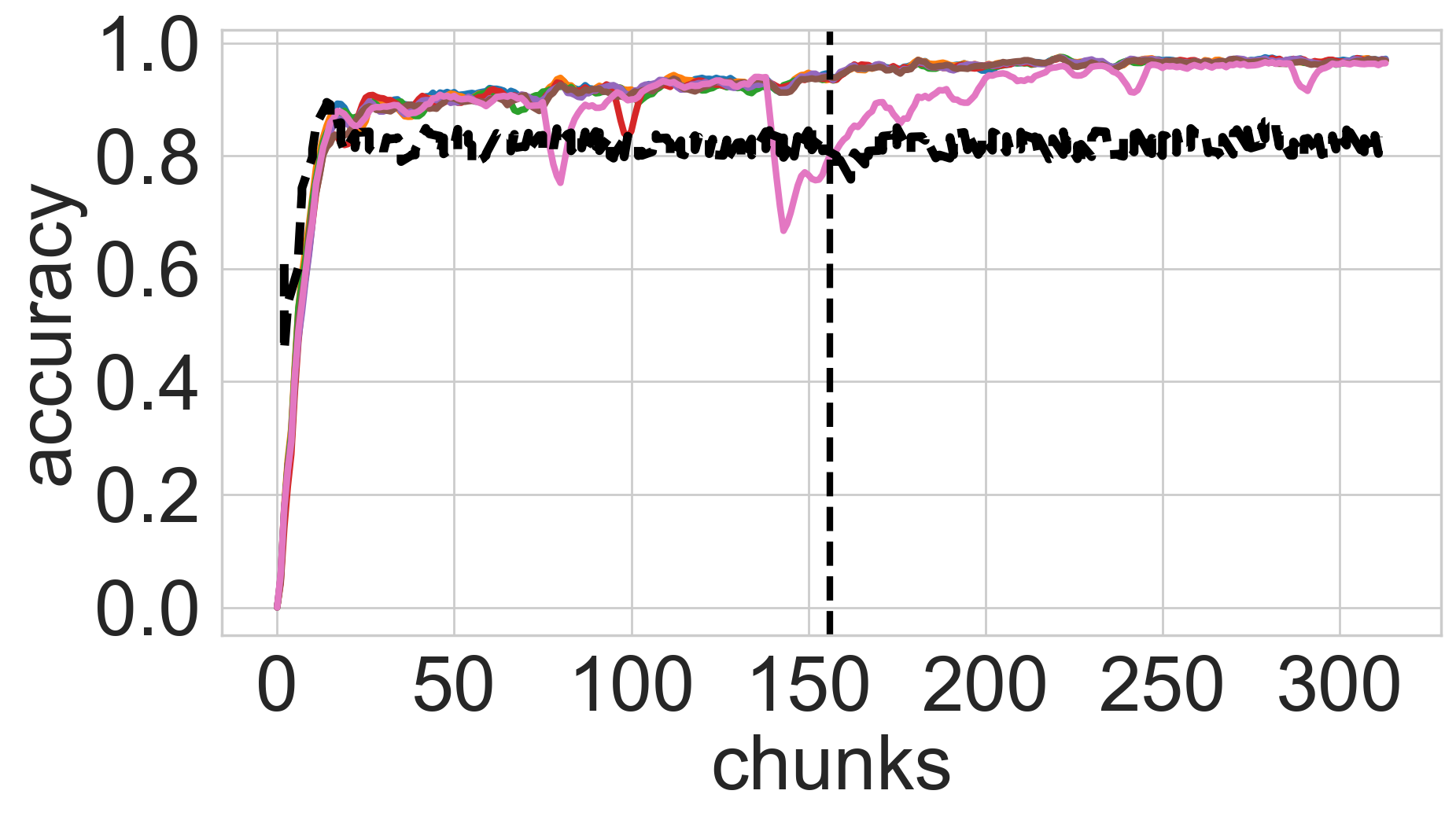}
\caption{}
\end{subfigure}
\hfill
\begin{subfigure}{0.31\textwidth}
\centering
\includegraphics[width=\linewidth]{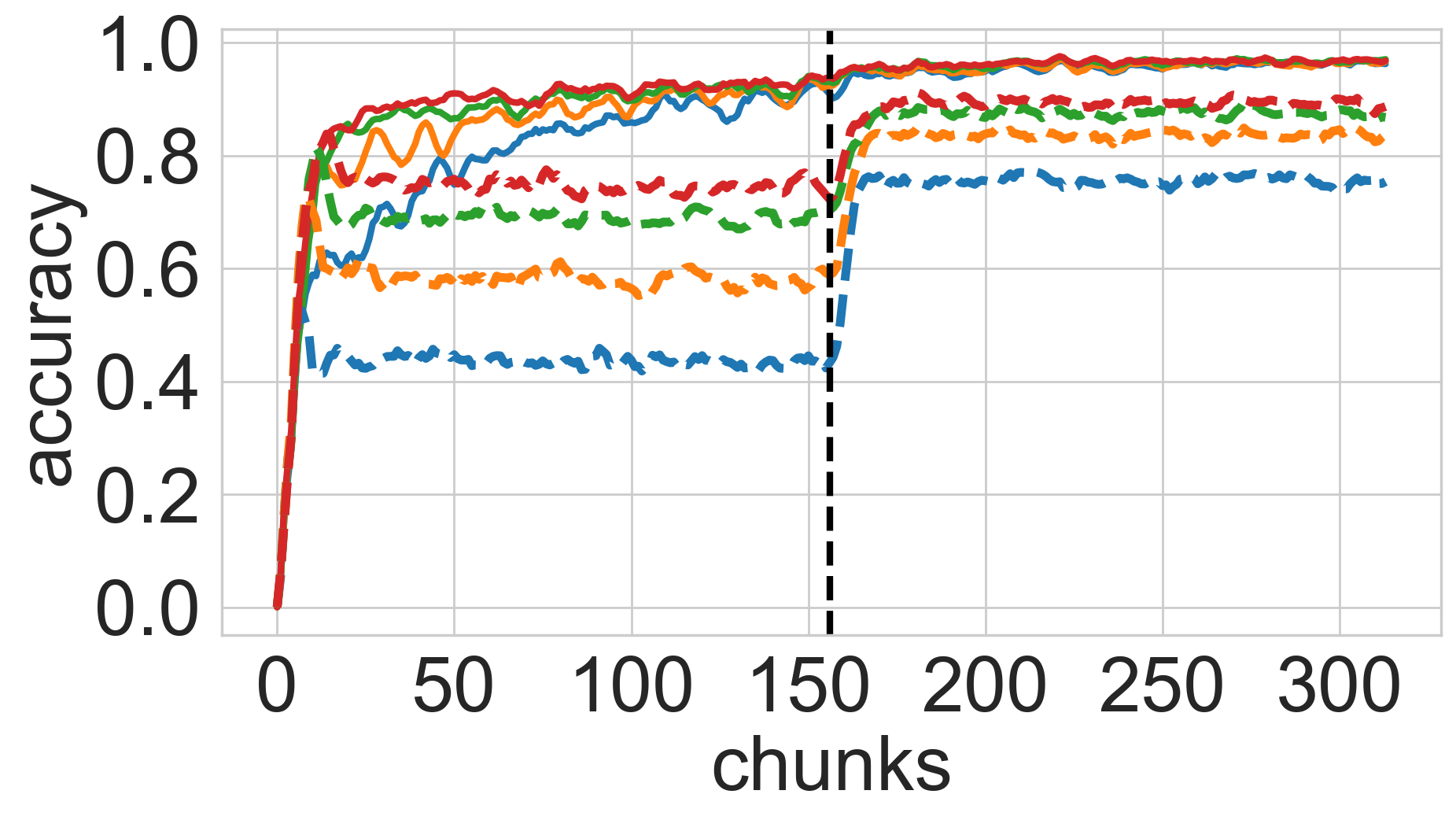}
\caption{}
\end{subfigure}

\caption{Comparison of \emph{SW} and \emph{UIL} on sudden MNIST data stream:  a, c, e) noise level 0.0 → 0.5, and  b, d, f) 0.5 → 0.0. }
\label{fig:mnist_accuracy}
\end{figure}


\begin{table}[!ht]
\caption{Comparisson time, data and recovery time efficiency on: A) semantic Fashion MNIST data stream [0, 2, 4, 7] → [1, 3, 5], B) semantic Fashion MNIST data stream [1, 3, 5] → [0, 2, 4, 7], C) sudden MNIST data stream (0.0 → 0.5), D) sudden MNIST data stream (0.5 → 0.0).}
\centering
\setlength{\arrayrulewidth}{0.4pt}
\begin{tabular}{lcccc!{\vrule width 1.2pt}cccc}
\hline
 & \multicolumn{4}{c}{\emph{SW}} & \multicolumn{4}{c}{\emph{UIL}} \\

 & \multicolumn{2}{c}{Fashion MNIST} & \multicolumn{2}{c}{MNIST}
 & \multicolumn{2}{c}{Fashion MNIST} & \multicolumn{2}{c}{MNIST} \\

 & A & B & C & D & A & B & C & D \\ 
\hline

Time [s per batch]      & 6.27 & 6.27 & 6.67 & 6.67 &  2.04 & 2.04 & 2.19 & 2.18 \\
Data [MB per batch]    & 196.47  & 196.49 & 198.47 & 198.46 & 188.02 & 188.03 & 188.02 & 188.03 \\
Recovery [chunks]       & 10.33 &  13.33 & 17.00 & 4.00 & 13.00 & 47.00 & 6.33 & 1.00 \\
Detoriation max [\%]    & 69.21 & 72.35 & 69.21 & 0.41 & 47.54 & 63.23 & 53.52 & 4.87 \\
Detoriation avg [\%]    & 7.923 & 7.68 & 16.72  & 17.73 & 7.87 & 9.68 & 4.36 & 2.70 \\

\hline
\end{tabular}
\label{tab:res}
\end{table}

\subsection{Results}

The results in Figs.~\ref{fig:fashion_accuracy} and~\ref{fig:mnist_accuracy} together with Table~\ref{tab:res} show that \emph{UIL} provides a favorable trade-off between adaptation quality and computational efficiency. Across both semantic drift on Fashion-MNIST and sudden noise-based drift on MNIST, \emph{UIL} remains competitive with the standard sliding-window baseline, while reducing the per-batch processing time by roughly a factor of three. In particular, SW requires about 6.27--6.67 seconds per batch, whereas \emph{UIL} requires only 2.04--2.19 seconds, confirming the practical efficiency gains predicted by the analytical study. \emph{UIL} also consistently uses slightly less memory per batch than SW, though this difference is modest compared to the runtime gain. It is also worth noting that, for most tested parameter values, \emph{UIL} remains stable, with instability observed only at very high learning rates. This observation is consistent with the conclusions of our earlier analysis. 

A more nuanced picture emerges from the recovery analysis. On the MNIST sudden-drift streams (Fig.~\ref{fig:mnist_accuracy}), \emph{UIL} often recovers faster than SW, requiring only 6.33 chunks versus 17.00 for the \(0.0 \rightarrow 0.5\) shift and 1.00 versus 4.00 for the \(0.5 \rightarrow 0.0\) shift. This suggests that targeted removal of outdated information can accelerate adaptation when the drift primarily affects low-level visual statistics. On Fashion-MNIST semantic drift (Fig.~\ref{fig:fashion_accuracy}), however, the behavior is less stable: \emph{UIL} remains faster computationally, but its recovery is slower than \emph{SW}, especially in scenario B, where recovery increases from 13.33 to 47.00 chunks. This indicates that semantic drift is more challenging for approximate unlearning, likely because the shift affects higher-level class structure rather than only input corruption. 

The deterioration measures in Table~\ref{tab:res} are consistent with this interpretation. \emph{UIL} generally reduces the worst deterioration relative to SW in the semantic Fashion-MNIST streams and in the MNIST \(0.0 \rightarrow 0.5\) setting, but this advantage is not uniform across all scenarios. Overall, the experiments suggest that \emph{UIL} is a strong alternative to sliding-window retraining when computational budget is limited, and it is particularly promising in sudden-drift scenarios where fast removal of obsolete information is more important than exact reconstruction of the full-window model. At the same time, the weaker recovery in one of the semantic drift settings highlights that the quality of the unlearning step remains a key factor in more structurally complex shifts. 



\section{Conclusion}

\noindent \textbf{Conclusions.} This paper presented \emph{UIL}, a new approach to concept drift that views nonstationary data streams as a \emph{task-free continual learning} problem and uses machine unlearning as a mechanism for simulated forgetting. By removing the influence of outdated data instead of retraining from scratch on a sliding window, \emph{UIL} offers a computationally efficient alternative to standard stream adaptation strategies. Our theoretical analysis showed that the resulting approximation error remains bounded, and the empirical study demonstrated that this approach can achieve competitive performance under sudden drift at substantially lower computational cost.

\noindent \textbf{Future works.} We envision developing this idea towards three directions: 

\noindent $\bullet$ \emph{Controlling unlearning with drift detectors.} Currently, \emph{UIL} employs the static unlearning rate. To prevent the accumulation of algorithmic noise during prolonged stationary phases, future work will integrate dynamic drift detectors. By monitoring distributional shifts—such as computing the Wasserstein distance or Kullback-Leibler (KL) divergence on the model's low-dimensional latent representations, we can dynamically scale the unlearning rate. This will allow the model to actively "pause" unlearning when drift stops.

\noindent $\bullet$  \emph{Unlearning for regularization-based continual learning.} To mitigate catastrophic forgetting variance ($\epsilon_{forgetting}$) during incremental learning of new chunks, we plan to focus on regularization techniques, such as Elastic Weight Consolidation (EWC)~\cite{Kirkpatrick:2017}. We aim to exploit a mathematical synergy within our framework: the approximate inverse Hessian matrix utilized during the unlearning step can serve as an empirical approximation of the Fisher Information Matrix required by EWC. This dual use of the curvature matrix would allow us to heavily penalize forgetting of stationary knowledge during the incremental update step, effectively neutralizing $\epsilon_{forgetting}$ with near-zero additional overhead.

\noindent $\bullet$ \emph{Unlearning for memory buffers.} Finally, we will explore the integration of strictly bounded, density-based memory buffers to seamlessly bridge the gap between active unlearning and traditional continual learning, providing an empirical anchor for the model's decision boundaries during complex drifts~\cite{Krawczyk:2017}.

\subsubsection*{Declaration. } We acknowledge the use of (i) Grammarly and Gemini to polish the selected parts of the manuscript, (ii) PaperBanana to prepare the figures.

\bibliographystyle{plain}
\bibliography{bibliography}

\end{document}